\def\year{2022}\relax
\documentclass[letterpaper]{article} 
\usepackage{aaai22}  
\usepackage{times}  
\usepackage{helvet}  
\usepackage{courier}  
\usepackage[hyphens]{url}  
\usepackage{graphicx} 
\usepackage{natbib}  
\usepackage{caption} 
\DeclareCaptionStyle{ruled}{labelfont=normalfont,labelsep=colon,strut=off} 
\usepackage{algorithm}
\usepackage{algorithmic}
\usepackage{amsmath,amssymb,amsfonts,amsthm}
\usepackage{cuted}
\usepackage{dblfloatfix}

\usepackage{newfloat}
\usepackage{listings}
\floatstyle{ruled}
\newfloat{listing}{tb}{lst}{}
\floatname{listing}{Listing}
\newtheorem{definition}{Definition}
\newtheorem{theorem}{Theorem}
\newtheorem{lemma}{Lemma}
\newtheorem{corollary}{Corollary}

\usepackage{etoolbox}
\makeatletter
\patchcmd{\@makecaption}
  {\scshape}
  {}
  {}
  {}
\makeatletter
\patchcmd{\@makecaption}
  {\\}
  {.\ }
  {}
  {}
\makeatother
\def\tablename{Table}

\begin{document}

\pdfinfo{
/Title Mathematical Justification of Hard Negative Mining via Isometric Approximation Theorem
/Author Albert Xu, Jhih-Yi Hsieh, Bhaskar Vundurthy, Eliana Cohen, Howie Choset, Lu Li
/TemplateVersion (2022.1)
}

\title{Mathematical Justification of Hard Negative Mining\\via Isometric Approximation Theorem}

\author{
   Albert Xu,
   Jhih-Yi Hsieh,
   Bhaskar Vundurthy,
   Eliana Cohen,
   Howie Choset,
   Lu Li
}
\setcounter{secnumdepth}{0} 

\maketitle

\begin{abstract}
In deep metric learning, the Triplet Loss has emerged as a popular method to learn many computer vision and natural language processing tasks such as facial recognition, object detection, and visual-semantic embeddings. One issue that plagues the Triplet Loss is network collapse, an undesirable phenomenon where the network projects the embeddings of all data onto a single point. 
Researchers predominately solve this problem by using triplet mining strategies. 
While hard negative mining is the most effective of these strategies, existing formulations lack strong theoretical justification for their empirical success.
In this paper, we utilize the mathematical theory of isometric approximation to show an equivalence between the Triplet Loss sampled by hard negative mining and an optimization problem that minimizes a Hausdorff-like distance between the neural network and its ideal counterpart function. This provides the theoretical justifications for hard negative mining's empirical efficacy.
In addition, our novel application of the isometric approximation theorem provides the groundwork for future forms of hard negative mining that avoid network collapse. 
Our theory can also be extended to analyze other Euclidean space-based metric learning methods like Ladder Loss or Contrastive Learning.
\end{abstract}


\section{Introduction}
Research in deep metric learning studies techniques for training deep neural networks to learn similarities and dissimilarities between data samples, typically by learning a distance metric via feature embeddings in $\mathbb R^n$. Most extensively, deep metric learning is used in face recognition \cite{schroff2015facenet, DBLP:journals/corr/LiuWYLRS17,hermans2017defense} and other computer vision tasks \cite{ DBLP:journals/corr/abs-2007-08176, chen2020simple} where there are an abundance of label values. 

Common deep metric learning techniques include contrastive loss \cite{HadsellR2006DRbL} and triplet loss \cite{schroff2015facenet}.
Moreover, each of these methods have variants to address specific applications.
SimCLR \cite{chen2020simple, chen2020big}, for example, is a recent contrastive loss variant designed to address unsupervised deep metric learning with state-of-the-art performance on ImageNet \cite{ILSVRC15}.
Ladder Loss \cite{DBLP:journals/corr/abs-1911-07528}, a generalized variant of triplet loss, improved upon existing methods for coherent visual-semantic embedding and has important applications in multiple visual and language understanding tasks \cite{DBLP:journals/corr/KarpathyJF14, DBLP:journals/corr/MaLL15, DBLP:journals/corr/VinyalsTBE14}.
Given the success of metric learning in a wide range of applications, we see value in investigating its underlying theories. 
In this paper, we present a theoretical framework which explains observed but previously unexplained behaviors of the Triplet Loss.

\section{Literature Review}
We choose to analyze Triplet Loss's underlying theory due to its strong dependence on the triplet selection strategy. This makes the Triplet Loss fickle to work with, as empirical results had shown that randomly sampling these triplets yielded unsatisfactory results.
On the other hand, successful triplet selection strategies like \textit{hard negative mining} can face issues like network collapse, a phenomenon where the network projects all data points onto a single point \cite{schroff2015facenet}, while more stable triplet selection strategies do not perform as well in practice \cite{hermans2017defense}.

In the original FaceNet paper, Schroff et al. find that with large batch sizes (thousands), hard negative mining lead to collapsed solutions. To address this, they instead used a strategy they called semi-hard mining \cite{schroff2015facenet}. On the other hand, Herman et al. find that with smaller batch sizes ($N=72$), the hardest mining strategy significantly out-performs other mining strategies, and does not suffer from collapsed solutions \cite{hermans2017defense}.
These seemingly contradictory results showcase the need for a theoretical framework to explain the theory of hard negative mining and the root cause of collapsed solutions.



There has been some prior literature investigating the phenomenon of network collapse. 
Xuan et al. show that hard negative mining leads to collapsed solutions by analyzing the gradients of a simplified neural network model \cite{xuan2020hard}.
However, they do not account for the many cases where hard negative mining does work. Levi et al. prove that, under a label randomization assumption, the globally optimal solution to the triplet loss necessarily exhibits network collapse \cite{levi2021rethinking}.
Rather than investigating functional hard mining strategies, Levi et al. instead suggest using the less effective easy positive mining to avoid network collapse.

In literature, there are plenty of claims that hard negative mining succeeds \cite{hermans2017defense,faghri2017vse++}, and numerous examples where it fails \cite{schroff2015facenet,ge2019visual,oh2016deep}. Our work explains why network collapse happens by using the theory of isometric approximation to better characterize the behavior of the Triplet Loss.

\section{Background and Definitions}
Establishing the notation used in the paper, let $\mathcal X$ be the data manifold and let $\mathcal Y$ be the classes with $|\mathcal Y|=c$ being the number of classes. Let $h:\mathcal X\to\mathcal Y$ be the true hypothesis function, or true labels of the data. Then the dataset consists of pairs $\{(x_k,y_k)\}_{k=1}^N$ with $x_k\in\mathcal X,y_k\in\mathcal Y$ and $y_k=h(x_k)$. We define the learned neural network as a function $f_\theta:\mathcal X\to\mathbb R^n$ which maps similar points in the data manifold $\mathcal X$ to similar points in $\mathbb R^n$.

As our paper focuses on metric learning, we define the similarity between embeddings to be the Euclidean distance $d(r_1,r_2)=||r_1-r_2||$ where $r_1, r_2 \in \mathbb{R}^n$. Further, we define the shorthand $d_\theta(x_1,x_2)=||f_\theta(x_1)-f_\theta(x_2)||$ where $x_1, x_2 \in \mathcal{X}$.


\subsection{Triplet Loss and Hard Negative Mining}

In this section, we discuss the Triplet Loss, one of the more successful approaches to supervised metric learning introduced by Schroff et al. \cite{schroff2015facenet}.
The Triplet Loss considers samples as triplets of data, composed of the anchor $(x\in\mathcal X)$, positive $(x^+)$, and negative $(x^-)$ samples, described in (\ref{eqn:triple}). The similarity relation (\ref{eqn:triple}a) requires that the anchor and positive samples must be of the same class, while the dissimilarity relation (\ref{eqn:triple}b) requires the anchor and negative must be of different classes. 
\begin{subequations}
    \begin{align}
        x^+\in\{x'\in\mathcal X | h(x)=h(x')\} \\
         x^-\in\{x'\in\mathcal X | h(x)\neq h(x')\} 
    \end{align}
    \label{eqn:triple}
\end{subequations}

Restating the objective of supervised metric learning, the embedding of the anchor sample must be closer to the positive than the negative for every triplet. An example of a satisfactory triplet is shown in Figure \ref{fig:tsep}. Formally, we express this relation via (\ref{eq:desprop}), where $\alpha$ is the margin term. 
\begin{equation}
    d_\theta(x,x^+) + \alpha \leq d_\theta(x,x^-)\ \  ~ \forall ~ x,x^+,x^- \in\mathcal X\label{eq:desprop}
\end{equation}
This leads to the definition of the Triplet Loss in (\ref{eq:ltrip}).
\begin{equation}
    \mathcal L_\textrm{Triplet} = \left[d_\theta(x,x^+) - d_\theta(x,x^-)+ \alpha \right]_+ \label{eq:ltrip}
\end{equation}

The function $[\ \cdot\ ]_+=\max(\cdot,0)$ zeroes negative values in order to ignore all the triplets that already satisfy the desired relation. In addition, as the margin $\alpha$ adds only a constant value to the loss function, its effect is negligible for small $\alpha$. Therefore, we will assume a zero value for the  margin ($\alpha=0$) for the remainder of this paper. 

\begin{definition} \label{defn:triplet}
\textbf{Triplet-Separated}. 
We refer to $m$ non-empty subsets $X^1,\cdots,$ $X^m\subset\mathbb R^n$ as \emph{\textbf{Triplet-Separated}} if for every $X^i$ and $X^j$ with $i\neq j$ we have
\begin{equation} ||x-y|| \leq ||x- z||\ \  \forall x,y\in X^i,\forall z\in X^j \label{eq:tsep1}
\end{equation}


This property can be extended to a function $f_\theta:\mathcal X\to\mathbb R^n$ by checking whether the embedding subsets $X_{f_\theta}^i$ are Triplet-Separated.
\begin{equation} X_{f_\theta}^i = \{f_\theta(x) | x \in \mathcal X, h(x)=i\} \end{equation}
\end{definition}

\begin{figure}
    \centering
    \includegraphics[width=.4\textwidth]{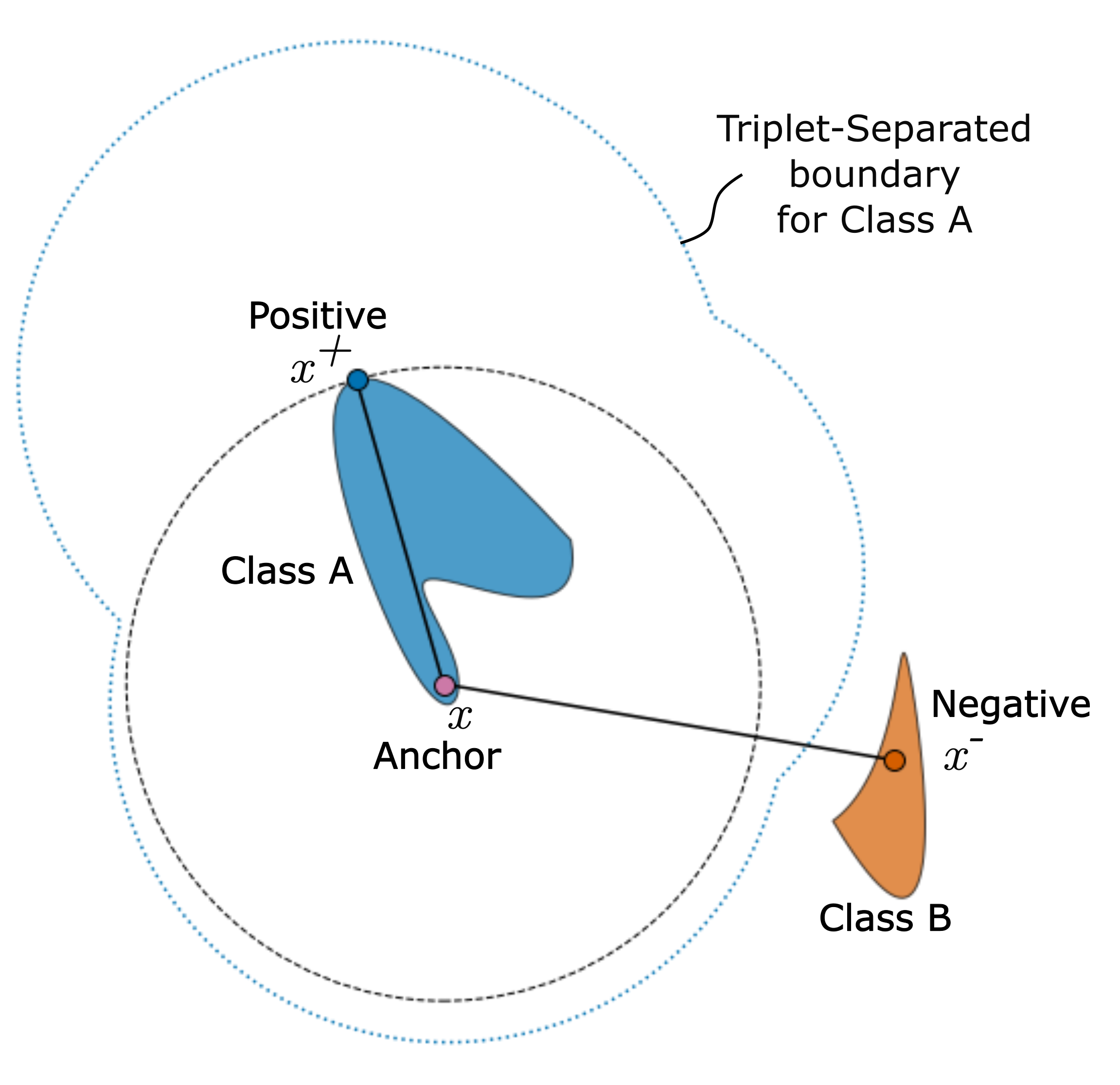}
    \caption{
    An example Anchor, Positive, and Negative triplet. 
    The blue dotted contour is the Triplet-Separated boundary for Class A. It is computed by considering inequality (\ref{eq:tsep1}) for all points in Class A. Because Class B is outside the Triplet-Separated boundary for Class A, the Triplet Loss for this example is zero.
    }
    \label{fig:tsep}
\end{figure}

It is worth noting that $\mathcal L_\textrm{Triplet}(f_\theta)=0$ if and only if $f_\theta$ is Triplet-Separated. An example of two Triplet-Separated sets is shown in Figure \ref{fig:tsep}.

As mentioned in the Literature Review section, the Triplet Loss relies heavily on its triplet mining strategy to achieve its performance for two popularly accepted reasons: First, enumerating all $O(N^3)$ triplets of data every iteration would be too computationally intensive to guarantee fast training. Second, improper sampling of triplets risks network collapse \cite{xuan2020hard}. Our work substantiates the use of hard negative mining, a successful triplet mining strategy, by characterizing conditions that lead to network collapse.

\subsection{Isometric Approximation}
We will present a novel application of the isometric approximation theorem in Euclidean subsets in order to mathematically justify hard negative mining.
The isometric approximation theorem primarily defines the behavior of near-isometries, or functions that are close to isometries, as given by \textbf{Definition \ref{defn:nearisometry2}}.



\begin{definition} \label{defn:nearisometry2}
\textbf{$\varepsilon$-nearisometry}. Let $X$ and $Y$ be real normed spaces. A function $f:A\to Y$ where $A\subset X$ is called an $\textbf{$\varepsilon$-\emph{nearisometry}}$ ($\varepsilon> 0$) if  
\begin{equation} \bigg|||f(x) - f(y)|| - ||x-y||\bigg| \leq \varepsilon, ~ \forall ~ x,y \in A \label{defn:nearisometry}\end{equation}
\end{definition}

In other words, an $\varepsilon$-nearisometry is a function that preserves the distance metric within $\varepsilon$. 
The isometric approximation theorem seeks to determine how close $f$ is to an isometry, say $U:X\to Y$, as given by (\ref{eqn:fU_isometry}).
Note $q_A(\varepsilon)$ is a function of $\varepsilon$ that is fixed for a given $A$ and is thus independent of $f$. Consequently, inequality (\ref{eqn:fU_isometry}) holds for all $\varepsilon> 0$ and all $\varepsilon$-nearisometries $f$.
\begin{equation} ||f(x) - U(x)|| \leq q_A(\varepsilon) ~ \forall ~ x\in A \label{eqn:fU_isometry} \end{equation}


Now consider the case where $X$ and $Y$ are n-dimensional Euclidean metric spaces, making $A\subset\mathbb R^n$. Then the following theorems and definitions \cite{vaisala2002isometric,vaisala2002survey,alestalo2001isometric} prove that $q_A(\varepsilon)$ is linear in $\varepsilon$ given a thickness condition on the set $A$.

\begin{definition}
\textbf{Thickness}. For each unit vector $e\in S^{n-1}$, define the projection $\pi_e:\mathbb R^n\to\mathbb R$ by the dot product $\pi_e(x) = x\cdot e$. The thickness of a bounded set $A$ is the number
\begin{subequations} 
        \begin{align}
            \theta(A) = \inf_{e\in S^{n-1}} \textrm{\emph{diam}}(\pi_e A) \\
            \textrm{\emph{where}} ~  ~ \textrm{\emph{diam}}(X) = \sup_{r_1,r_2\in X} ||r_1-r_2||
        \end{align}
\end{subequations}
\end{definition}

\begin{theorem}[From \textbf{Theorem 3.3} \cite{alestalo2001isometric}]
\label{thm:ciap}
 Suppose that $0<q\leq 1$ and $A\subset \mathbb R^n$ is a compact set with $\theta(A)\geq q\ \textrm{diam}(A)$. Let $f:A\to\mathbb R^n$ be an $\varepsilon$-nearisometry. Then there is an isometry $U:\mathbb R^n\to\mathbb R^n$ such that
\begin{equation} ||f(x)-U(x)|| \leq c_n\varepsilon /q\ \  \forall x\in A \end{equation}

\noindent with $c_n$ depending only on dimension.
\end{theorem}


As this property depends entirely on the set $A$, we call Theorem \ref{thm:ciap} the \textbf{$c$-Isometric Approximation Property} (c-IAP) on set $A$ with $c=c_n$ diam($A)/\theta(A)$.


\section{Theory and Proofs} \label{sec:theory}
\subsection{Overview and Problem Setup}
From the background and definitions, the goal of the Triplet Loss is to learn a function $f_\theta$ such that the induced distance metric $d_\theta$ satisfies the property in (\ref{eq:desprop}). However, rather than starting with the Triplet Loss then introducing hard negative mining as a method to correct the weaknesses of the Triplet Loss, we instead derive hard negative mining and the Triplet Loss together, using the Hausdorff-like distance $d_\textrm{haus}$ as a starting point.







\subsection{Hausdorff-Like Distance $d_\textrm{haus}$}

Reiterating the training objective from the problem setup, we aim to learn a function $f_\theta$ that is Triplet-Separated (\textbf{Definition \ref{defn:triplet}}). We restate this problem as a distance minimization problem, and prove that it is equivalent to hard negative mining with the Triplet Loss.

First we construct set of all functions $f:\mathcal X\to\mathbb R^n$ that are Triplet-Separated and denote it with $\mathcal F_{TS}$.
We next construct the Hausdorff-like distance metric (denoted by $d_\textrm{haus}$) between these functions that compares the embedding subsets via the Hausdorff distance metric $d_H$. 

\begin{gather}
    X_f^i = \{f(x) | x\in\mathcal X, h(x)=i\} \\
    d_\textrm{haus}(f_1, f_2) = \max_{i\in\mathcal Y} d_H(X_{f_1}^i, X_{f_2}^i) \label{eq:defn-dhaus}
\end{gather}

One way to solve metric learning is to find the closest $f_\theta$ to any function in $\mathcal F_{TS}$ as indicated by (\ref{eq:dFTShaus}).
\begin{equation}
    d_\textrm{haus}(f_\theta,\mathcal F_{TS}) = \inf_{f_\textrm{haus}^*\in\mathcal F_{TS}} d_\textrm{haus}(f_\theta, f_\textrm{haus}^*) \label{eq:dFTShaus}
\end{equation}

We claim that the Triplet Loss with hard negative mining is equivalent to minimizing $d_\textrm{haus}(f_\theta,\mathcal F_{TS})$ within a constant factor (see \textbf{Corollary \ref{coro:1}}).

\subsection{Isometric Approximation Applied to $d_\textrm{haus}$}
In this section, we present \textbf{Theorem \ref{thm:hausiso}} to show that minimizing the Hausdorff-like distance is equivalent to minimizing a discrepancy in distance metrics, referred to as the \textbf{isometric error} (\textbf{Definition \ref{defn:diso}}).




\begin{definition} \label{defn:diso}
\textbf{isometric error}. 
For two functions $f,g:\mathcal X\to\mathbb R^n$, we define the \textbf{\emph{isometric error}} $d_\textrm{iso}$ to be the maximum discrepancy between their distance metrics.
\begin{gather}
  \begin{aligned}
    d_\textrm{iso}(f,g) = 
    \sup_{x_1,x_2\in\mathcal X}\bigg| ||f(x_1) - f(x_2)|| - ||g(x_1) - g(x_2)|| \bigg|
  \end{aligned}
\end{gather}

\end{definition}


Similar to (\ref{eq:dFTShaus}), we extend the definition of \textbf{isometric error} to $d_\textrm{iso}(f_\theta,\mathcal F_{TS})$ as follows:
\begin{equation}
    d_\textrm{iso}(f_\theta, \mathcal F_{TS}) = \inf_{f_\textrm{iso}^*\in\mathcal F_{TS}} d_\textrm{iso}(f_\theta, f_\textrm{iso}^*)
\end{equation}

\begin{lemma} \label{lemma:1}
If $d_\textrm{iso}(f,g)=\varepsilon$ and $\theta(f(\mathcal X)) \geq q$, then there is a function $U$ isometric to $f$ such that:
\begin{equation} \label{eq:lemma1}
    ||g(x) - U(x)|| \leq c_n\varepsilon / q ~ \forall x\in\mathcal X
\end{equation}
\end{lemma}
\begin{proof} 
If $f$ is invertible, then $g f^{-1}$ is a function $\mathbb R^n\to\mathbb R^n$. $gf^{-1}$ is an $\varepsilon$-nearisometry because $d_\textrm{iso}(f,g)=\varepsilon$. Then if $\theta(f(\mathcal X)) \geq q$, the conditions for \textbf{Theorem 1} are satisfied, so there exists an isometry $U_1:\mathbb R^n\to\mathbb R^n$
\begin{equation}
    ||gf^{-1}(x) - U_1(x)|| \leq c_n\varepsilon / q ~ \forall x\in f(\mathcal X)
\end{equation}

Then
\begin{equation} \label{eq:lemma1p2}
    ||g(x) - U_1(f(x)) || \leq c_n\varepsilon / q ~ \forall x\in\mathcal X
\end{equation}

Therefore if $f$ is invertible, (\ref{eq:lemma1}) holds with $U=U_1f$.

If $f$ is not invertible, then there exists $x_1\neq x_2\in\mathcal X$ such that $f(x_1)=f(x_2)$. We divide the elements of $\mathcal X$ into subsets $X^\dagger$ and $X'$ such that $f$ is invertible on $X^\dagger$, $f(\mathcal X^\dagger)=f(\mathcal X)$, and $d_\textrm{iso}$ is unchanged on $\mathcal X^\dagger$. Consequently, (\ref{eq:lemma1p2}) holds on $\mathcal X^\dagger$.

Moving our attention to $X'$, for all $x'\in X'$ there exists $x^\dagger\in X^\dagger$ such that $f(x')=f(x^\dagger)$. Then because $d_\textrm{iso}$ is unchanged on $X^\dagger$, $||f(x')-g(x')|| \leq ||f(x^\dagger) - g(x^\dagger)|| \leq c_n\varepsilon/q$. Therefore (\ref{eq:lemma1}) holds for $f$ and $g$ on $\mathcal X$.
\end{proof}

\begin{theorem} \label{thm:hausiso}
$d_\textrm{haus}(f_\theta,\mathcal F_{TS})$ and $d_\textrm{iso}(f_\theta,\mathcal F_{TS})$ upper bound each other within a linear factor for all $f_\theta$ with some minimum thickness $\theta_{*}$.
\end{theorem}

\begin{proof}

We first prove that $d_\textrm{iso}$ upper bounds $d_\textrm{haus}$. To this end, fix the minimizing function $f^*_\textrm{iso}$ in the following expression:
\begin{equation} d_\textrm{iso}(f_\theta, \mathcal F_{TS}) = \inf_{f^*_\textrm{iso}\in\mathcal F_{TS}} d_\textrm{iso}(f_\theta, f^*_\textrm{iso}) \end{equation}

From \textbf{Lemma \ref{lemma:1}} we have that:
\begin{equation} \sup_{x\in\mathcal X} ||f_\theta(x) - f^*_\textrm{iso}(x)|| \leq c\ d_\textrm{iso}(f_\theta,f^*_\textrm{iso}) \label{eq:thm2p0} \end{equation}

\noindent with $c=c_n/\theta_*$. From the definition of Hausdorff-like distance (\ref{eq:defn-dhaus}) we have (\ref{eq:thm2p1}), and from (\ref{eq:dFTShaus}) we have (\ref{eq:thm2p2}):
\begin{gather}
    d_\textrm{haus}(f_\theta,f^*_\textrm{iso}) \leq \sup_{x\in\mathcal X} ||f_\theta(x) - f^*_\textrm{iso}(x)|| \label{eq:thm2p1} \\
    d_\textrm{haus}(f_\theta,\mathcal F_{TS}) \leq d_\textrm{haus}(f_\theta,f^*_\textrm{iso}) \label{eq:thm2p2} \\
    d_\textrm{haus}(f_\theta,\mathcal F_{TS}) \leq c\ d_\textrm{iso}(f_\theta,\mathcal F_{TS}) \label{eq:thm2p3}
\end{gather}

(\ref{eq:thm2p3}) follows from (\ref{eq:thm2p0}-\ref{eq:thm2p2}), proving that $d_\textrm{iso}$ upper bounds $d_\textrm{haus}$ within a constant factor of $c$.

For the converse claim that $d_\textrm{haus}$ upper bounds $d_\textrm{iso}$, we once again fix the $f^*_\textrm{haus}$ that minimizes the following expression:
\begin{equation}
d_\textrm{haus}(f_\theta,\mathcal F_{TS}) = \sup_{x\in\mathcal X} || f_\theta(x) - f^*_\textrm{haus}(x)|| \label{eq:thm2p3.5}
\end{equation}

Next, for the four points $f_\theta(x_1)$, $f_\theta(x_2)$, $f^*_\textrm{haus}(x_1)$, and $f^*_\textrm{haus}(x_2)$, apply the triangle inequality via (\ref{eq:thm2p4}) to get (\ref{eq:thm2p5}).
\begin{align}
    ||f_\theta(x_1) - f_\theta(x_2)|| &\leq 
    \left(\begin{aligned}
     & ||f_\theta(x_1) - f^*_\textrm{haus}(x_1)|| +\\
     & ||f^*_\textrm{haus}(x_1) - f^*_\textrm{haus}(x_2)|| +\\
     & ||f^*_\textrm{haus}(x_2) - f_\theta(x_2)||
    \end{aligned}\right) \label{eq:thm2p4} \\
    &\leq \left(\begin{aligned}
     & ||f^*_\textrm{haus}(x_1) - f^*_\textrm{haus}(x_2)|| +\\
     & 2\sup_{x\in\mathcal X} || f_\theta(x) - f^*_\textrm{haus}(x)|| 
    \end{aligned}\right) \label{eq:thm2p5}
\end{align}


It is worth noting that (\ref{eq:thm2p5}) holds for all $x_1,x_2\in\mathcal X$. Furthermore, we can swap $f_\theta$ and $f^*_\textrm{haus}$ in (\ref{eq:thm2p5}) and use (\ref{eq:thm2p3.5}) to get (\ref{eq:thm2p6}) and thus (\ref{eq:thm2p7}).

\begin{gather}
\begin{alignedat}[t]{2}
    &d_\textrm{iso}(f_\theta, f^*_\textrm{haus}) = &&  \\
    &\sup_{x_1,x_2\in \mathcal X} \bigg|||f_\theta(x_1)-f_\theta(x_2)|| - ||&&f^*_\textrm{haus}(x_1)-f^*_\textrm{haus}(x_2)|| \bigg| \leq \\
    & && 2d_\textrm{haus}(f_\theta,\mathcal F_{TS})
\end{alignedat} \label{eq:thm2p6} \\
d_\textrm{iso}(f_\theta,\mathcal F_{TS}) \leq d_\textrm{iso}(f_\theta, f^*_\textrm{haus}) \leq 2d_\textrm{haus}(f_\theta,\mathcal F_{TS}) \label{eq:thm2p7}
\end{gather}

(\ref{eq:thm2p7}) proves that $d_\textrm{haus}$ upper bounds $d_\textrm{iso}$ within a constant factor of $2$.
\end{proof}



Theorem \ref{thm:hausiso} shows that $d_\textrm{haus}$ and $d_\textrm{iso}$ are exchangeable as minimization objectives because they upper bound each other within linear factors. 
Now that we have rewritten the minimization objective as a difference of two distance functions, we can derive the Triplet Loss.


\subsection{Recovering the Triplet Loss}

In this section, we will prove that $d_\textrm{iso}$ (\textbf{Definition \ref{defn:diso}}) is equivalent to the Triplet Loss sampled by hard negative mining.

\begin{theorem} \label{thm:disotlhnm}
The Triplet Loss sampled by hard negative mining and the isometric error $d_\textrm{iso}$ upper bound each other within a linear factor.
\end{theorem}

We present the proof for \textbf{Theorem \ref{thm:disotlhnm}} in Appendix A.

From \textbf{Theorems \ref{thm:hausiso}} and \textbf{\ref{thm:disotlhnm}} we have \textbf{Corollary \ref{coro:1}}.

\begin{corollary} \label{coro:1}
The optimal solution to the Triplet Loss sampled by hard negative mining is equivalent to the optimal solution to $d_\textrm{haus}(f_\theta,\mathcal F_{TS})$ within a constant factor.
\end{corollary}
\begin{proof}
The proof follows from Theorems \ref{thm:hausiso} and \ref{thm:disotlhnm}, where we show that $d_\textrm{haus}$, $d_\textrm{iso}$, and Triplet Loss sampled by hard negative mining upper and lower bound each other by constant factors. Consequently, the optimal solution to Triplet Loss sampled by hard negative mining, and to $d_\textrm{haus}(f_\theta,\mathcal F_{TS})$, are equivalent within a constant factor.
\end{proof}



\section{Illustrative Examples}
In this section, we illustrate the key ideas of the previous section's theorems by using a toy example with $N=5$ points and embedding dimension $n=2$.
As we will illustrate the equivalence between the Triplet Loss with the Hausdorff-like distance and isometric error, we can visualize the embedding points without any underlying data or neural network. See Figure \ref{fig:toy5a} for the toy example setup.


\begin{figure}[t!]
    \centering
    \includegraphics[width=.4\textwidth]{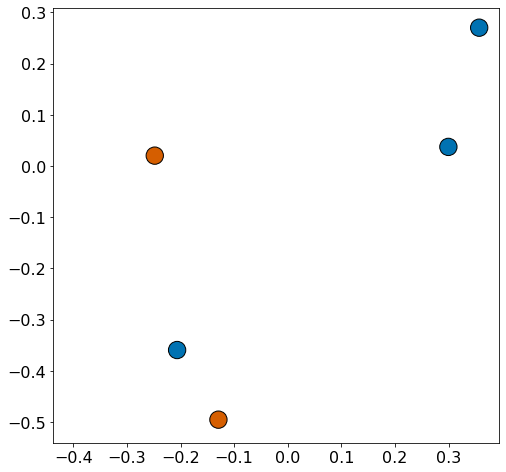}
    \caption{The setup for our toy example is a dataset of five arbitrary points in $\mathbb R^2$, divided into two classes (red and blue).}
    \label{fig:toy5a}
\end{figure}

\begin{figure}[t!]
    \centering
    \includegraphics[width=.4\textwidth]{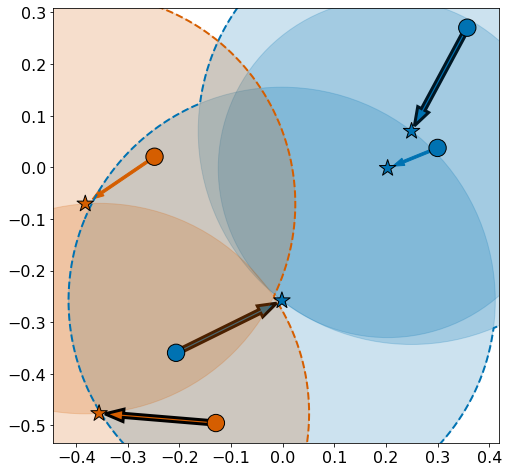}
    \caption{Illustration of $d_\textrm{haus}(f_\theta,\mathcal F_{TS})$. Using the toy example shown in Figure \ref{fig:toy5a}, we compute a $f^*_{\textrm{haus}}$ that minimizes $d_\textrm{haus}(f_\theta,\mathcal F_{TS})$. Arrows represent the function $f^*_{\textrm{haus}}$ and stars indicate the embedding points for each class.
    The red and blue star sets are Triplet-Separated because they lie outside the other's Triplet-Separated boundary, indicated by the dashed colored border.
    The three most important contributors to $d_\textrm{haus}(f_\theta,\mathcal F_{TS})$ are marked with black arrows.
    }
    \label{fig:toy5b}
\end{figure}

First, we will visualize $d_\textrm{haus}(f_\theta,\mathcal F_{TS})$ in Figure \ref{fig:toy5b}. The numerical value of $d_\textrm{haus}(f_\theta,\mathcal F_{TS})$ is determined by the maximum length of the arrows (see caption to Figure \ref{fig:toy5b}), which is marked in the figure with black outlines. Here, we compute the ideal $f^*_\textrm{haus}$ by optimizing the embedding points.


\begin{figure}
    \centering
    \includegraphics[width=.45\textwidth]{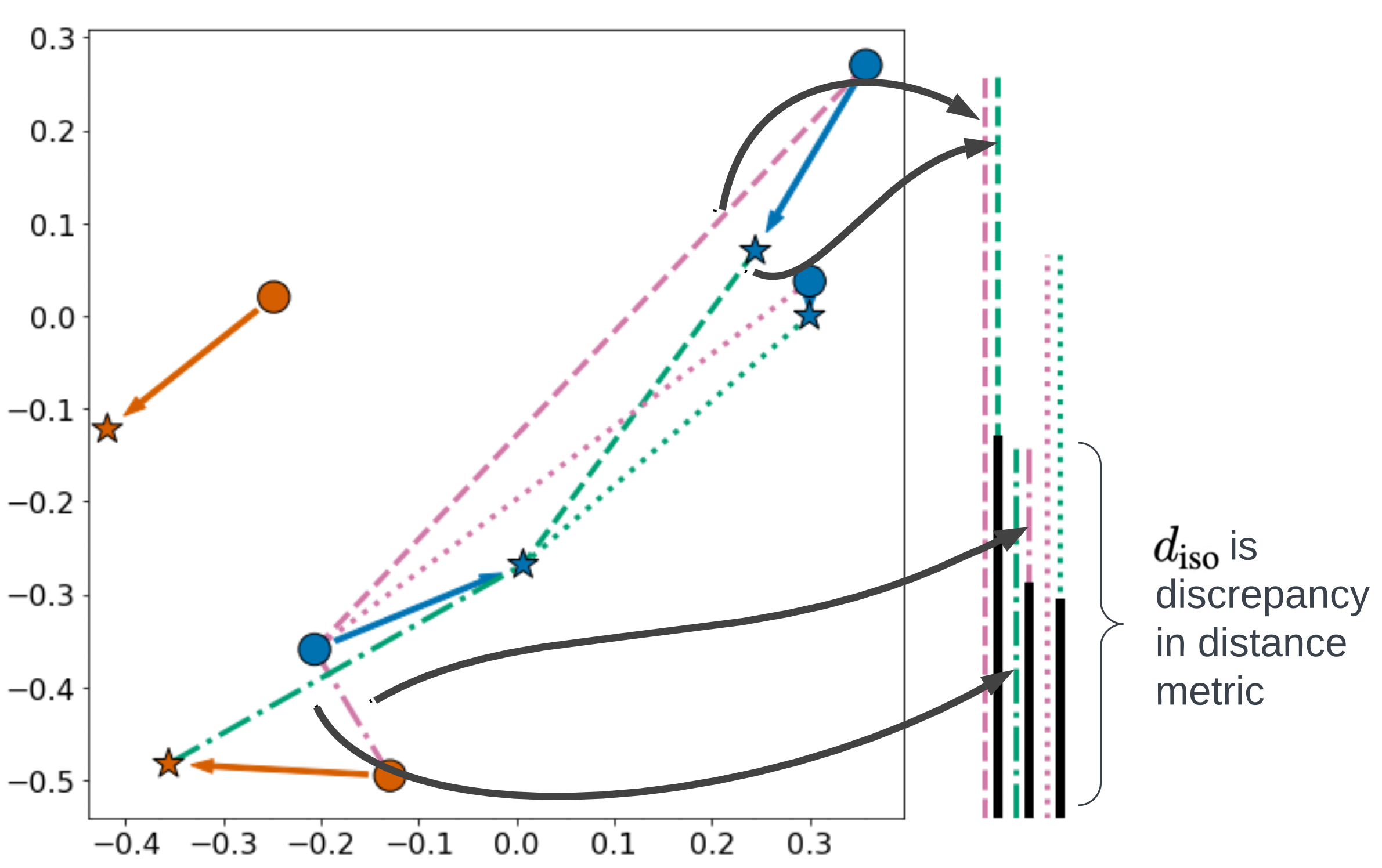}
    \caption{Illustration of $d_\textrm{iso}(f_\theta,\mathcal F_{TS})$. Using the toy example shown in Figure \ref{fig:toy5a}, we compute a $f^*_{\textrm{iso}}$ that minimizes $d_\textrm{iso}(f_\theta,\mathcal F_{TS})$
    Here we compare the distance metric, subtracting the distance between two points under $f_\theta$ (circles) and the distance under $f^*$ (stars), to compute $d_\textrm{iso}$ as shown by the black vertical bar on the right.}
    \label{fig:toy5c}
\end{figure}

Figure \ref{fig:toy5c} illustrates $d_\textrm{iso}(f_\theta,\mathcal F_{TS})$, which measures the discrepancy in distance metric. Note that the $f^*_\textrm{iso}$ that minimizes $d_\textrm{iso}(f_\theta,f^*_\textrm{iso})$ is not necessarily the same as $f^*_\textrm{haus}$.
Revisiting the second part of the proof for \textbf{Theorem \ref{thm:hausiso}}, $d_\textrm{haus}(f_\theta,\mathcal F_{TS})$ is lower bounded by $0.5d_\textrm{iso}(f_\theta,\mathcal F_{TS})$ and upper bounded by $cd_\textrm{iso}(f_\theta,\mathcal F_{TS})$. For this specific toy example, we calculate the constant factor error to be $c=0.53$, making $d_\textrm{haus}(f_\theta,\mathcal F_{TS})$ an almost linear factor of $d_\textrm{iso}(f_\theta,\mathcal F_{TS})$. This essentially illustrates \textbf{Theorem \ref{thm:hausiso}}.


\begin{figure}[t!]
    \centering
    \includegraphics[width=.4\textwidth]{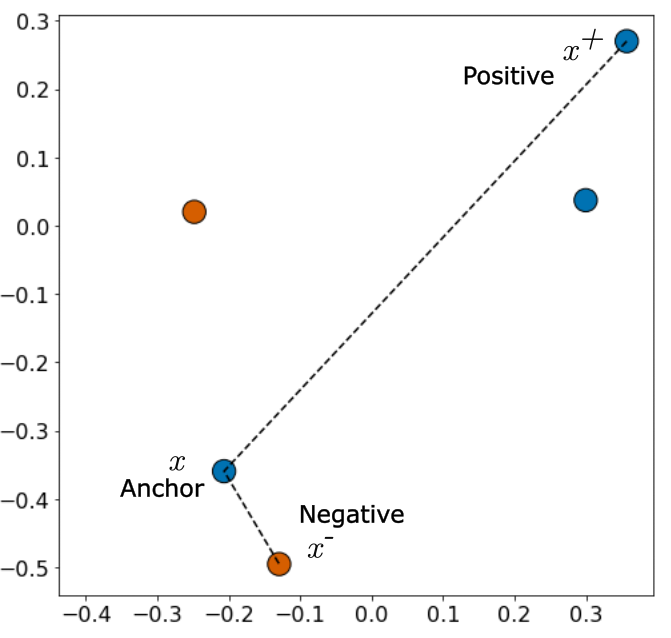}
    \caption{Illustration of the Triplet Loss sampled by hard negative mining. Using the toy example shown in Figure \ref{fig:toy5a}, we take the triplet (anchor, positive, negative) that maximizes the Triplet Loss.}
    \label{fig:toy5d}
\end{figure}

Next, we show the Triplet Loss sampled by hard negative mining in Figure \ref{fig:toy5d}. The equivalence proved by \textbf{Theorem \ref{thm:disotlhnm}} is shown by comparing Figure \ref{fig:toy5d} against Figure \ref{fig:toy5c}, as the triplet selected by hard negative mining corresponds with the same three points with the largest discrepancies in distance metric.
Through Figures \ref{fig:toy5b}, \ref{fig:toy5c}, and \ref{fig:toy5d}, we have a visualization of the statement and proof of \textbf{Corollary \ref{coro:1}}.

\section{Discussion}
\subsection{Novel Insights on Network Collapse}
As mentioned in the Literature Review section, current literature observes that hard negative mining results in network collapse inconsistently. We propose a theory that explains hard negative mining's intermittent behavior.

We hypothesize that network collapse happens when the $f^*$ that minimizes $d_\textrm{haus}(f_\theta,\mathcal F_{TS})$ is a collapsed function, or when the ``nearest'' function maps all the data points onto a much smaller subset. An example of this effect can be seen in figure \ref{fig:toy20}, where we have 20 random data points with a collapsed $f^*$.

\begin{figure}
    \centering
    \includegraphics[width=.4\textwidth]{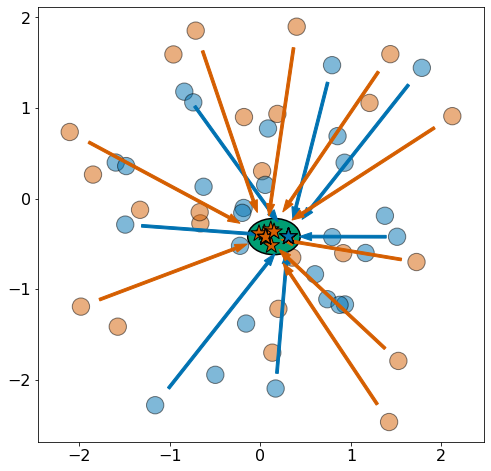}
    \caption{Illustration of how hard negative mining leads to collapsed solutions given $N=20$ samples and embedding dimension $n=2$. We find $f^*$ that minimizes $d_\textrm{haus}(f_\theta,\mathcal F_{TS})$ and observe that the embedding points (stars) collapse into a much smaller subset, marked by the ellipse in green.
    }
    \label{fig:toy20}
\end{figure}

We observe that when the number of samples $N$ is much greater than the embedding dimension $n$, the ideal counterpart $f^*$ is much more likely to collapse. On the contrary, when $n<N$, the network is much less likely collapse. Therefore when training a network with a batch-hard negative mining strategy where the network learns on the hardest triplet in a single batch, we expect the batch-hard fails when $N\gg n$, and does not fail otherwise.

We further support this hypothesis by examining prior publications that use hard negative mining.
Herman et al. \cite{hermans2017defense} find hard negative mining works with a batch size of $N=72$ and embedding dimension $n=128$. On the other hand, Schroff et al. \cite{schroff2015facenet} find that hard negative mining fails when they use thousands of samples per batch with embedding dimension $n=128$.

\subsection{Limitations and Future Work}
In (\ref{eq:desprop}) of the Background and Definitions section, we assumed the margin term $\alpha$ to be zero. This assumption may not be valid when $\alpha$ is large enough to affect the optimal solution to the Triplet Loss sampled by hard negative mining. In the future, we plan to further study the effect of the margin $\alpha$ on the Triplet Loss.




Additionally, we intend to study methods to avoid network collapse. As illustrated in Figure \ref{fig:toy20}, network collapse occurs because the $f^*$ that minimizes $d_\textrm{haus}(f_\theta,\mathcal F_{TS})$ is a collapsed function. To avoid network collapse, we would restrict the set $\mathcal F_{TS}$ to disallow collapsed functions. This would necessitate deriving a new set of equations that correctly utilize the restricted set.

\section{Conclusion}
In this paper, we apply the isometric approximation theorem to prove that the Triplet Loss sampled by hard negative mining is equivalent to minimizing a Hausdorff-like distance.
This mathematical foundation produces new insights into hard negative mining. In particular, it explains network collapse, a phenomenon that prior theories were unable to fully explain.



With these insights, we provide the groundwork for future forms of hard negative mining that avoid network collapse.
Further, as the theory of isometric approximation is independent of the Triplet Loss, it can be applied to any system utilizing the Euclidean metric or cosine similarity on a sphere. Therefore, this theory can be extended to analyze other metric learning methods like Ladder Loss or Contrastive Learning.

Through this and future work, we intend to leverage the power of mathematics to explain the fundamental principles of `black-box' machine learning approaches. Categorizing this previously undefined behavior forms new opportunities to strengthen modern machine learning and artificial intelligence research.

\bibliography{bibliography}

\begin{thebibliography}{20}
\providecommand{\natexlab}[1]{#1}

\bibitem[{Alestalo, Trotsenko, and
  V{\"a}is{\"a}l{\"a}(2001)}]{alestalo2001isometric}
Alestalo, P.; Trotsenko, D.; and V{\"a}is{\"a}l{\"a}, J. 2001.
\newblock Isometric approximation.
\newblock \emph{Israel Journal of Mathematics}, 125(1): 61--82.

\bibitem[{Chen et~al.(2020{\natexlab{a}})Chen, Kornblith, Norouzi, and
  Hinton}]{chen2020simple}
Chen, T.; Kornblith, S.; Norouzi, M.; and Hinton, G. 2020{\natexlab{a}}.
\newblock A simple framework for contrastive learning of visual
  representations.
\newblock In \emph{International conference on machine learning}, 1597--1607.
  PMLR.

\bibitem[{Chen et~al.(2020{\natexlab{b}})Chen, Kornblith, Swersky, Norouzi, and
  Hinton}]{chen2020big}
Chen, T.; Kornblith, S.; Swersky, K.; Norouzi, M.; and Hinton, G.~E.
  2020{\natexlab{b}}.
\newblock Big self-supervised models are strong semi-supervised learners.
\newblock \emph{Advances in neural information processing systems}, 33:
  22243--22255.

\bibitem[{Faghri et~al.(2017)Faghri, Fleet, Kiros, and
  Fidler}]{faghri2017vse++}
Faghri, F.; Fleet, D.~J.; Kiros, J.~R.; and Fidler, S. 2017.
\newblock Vse++: Improving visual-semantic embeddings with hard negatives.
\newblock \emph{arXiv preprint arXiv:1707.05612}.

\bibitem[{Ge, Gao, and Liu(2019)}]{ge2019visual}
Ge, J.; Gao, G.; and Liu, Z. 2019.
\newblock Visual-textual association with hardest and semi-hard negative pairs
  mining for person search.
\newblock \emph{arXiv preprint arXiv:1912.03083}.

\bibitem[{Hadsell, Chopra, and LeCun(2006)}]{HadsellR2006DRbL}
Hadsell, R.; Chopra, S.; and LeCun, Y. 2006.
\newblock Dimensionality Reduction by Learning an Invariant Mapping.
\newblock In \emph{2006 IEEE Computer Society Conference on Computer Vision and
  Pattern Recognition (CVPR'06)}, volume~2, 1735--1742. IEEE.
\newblock ISBN 9780769525976.

\bibitem[{Hermans, Beyer, and Leibe(2017)}]{hermans2017defense}
Hermans, A.; Beyer, L.; and Leibe, B. 2017.
\newblock In defense of the triplet loss for person re-identification.
\newblock \emph{arXiv preprint arXiv:1703.07737}.

\bibitem[{Karpathy, Joulin, and
  Fei{-}Fei(2014)}]{DBLP:journals/corr/KarpathyJF14}
Karpathy, A.; Joulin, A.; and Fei{-}Fei, L. 2014.
\newblock Deep Fragment Embeddings for Bidirectional Image Sentence Mapping.
\newblock \emph{CoRR}, abs/1406.5679.

\bibitem[{Levi et~al.(2021)Levi, Xiao, Wang, and Darrell}]{levi2021rethinking}
Levi, E.; Xiao, T.; Wang, X.; and Darrell, T. 2021.
\newblock Rethinking preventing class-collapsing in metric learning with
  margin-based losses.
\newblock In \emph{Proceedings of the IEEE/CVF International Conference on
  Computer Vision}, 10316--10325.

\bibitem[{Liu et~al.(2017)Liu, Wen, Yu, Li, Raj, and
  Song}]{DBLP:journals/corr/LiuWYLRS17}
Liu, W.; Wen, Y.; Yu, Z.; Li, M.; Raj, B.; and Song, L. 2017.
\newblock SphereFace: Deep Hypersphere Embedding for Face Recognition.
\newblock \emph{CoRR}, abs/1704.08063.

\bibitem[{Ma, Lu, and Li(2015)}]{DBLP:journals/corr/MaLL15}
Ma, L.; Lu, Z.; and Li, H. 2015.
\newblock Learning to Answer Questions From Image using Convolutional Neural
  Network.
\newblock \emph{CoRR}, abs/1506.00333.

\bibitem[{Oh~Song et~al.(2016)Oh~Song, Xiang, Jegelka, and
  Savarese}]{oh2016deep}
Oh~Song, H.; Xiang, Y.; Jegelka, S.; and Savarese, S. 2016.
\newblock Deep metric learning via lifted structured feature embedding.
\newblock In \emph{Proceedings of the IEEE conference on computer vision and
  pattern recognition}, 4004--4012.

\bibitem[{Russakovsky et~al.(2015)Russakovsky, Deng, Su, Krause, Satheesh, Ma,
  Huang, Karpathy, Khosla, Bernstein, Berg, and Fei-Fei}]{ILSVRC15}
Russakovsky, O.; Deng, J.; Su, H.; Krause, J.; Satheesh, S.; Ma, S.; Huang, Z.;
  Karpathy, A.; Khosla, A.; Bernstein, M.; Berg, A.~C.; and Fei-Fei, L. 2015.
\newblock {ImageNet Large Scale Visual Recognition Challenge}.
\newblock \emph{International Journal of Computer Vision (IJCV)}, 115(3):
  211--252.

\bibitem[{Schroff, Kalenichenko, and Philbin(2015)}]{schroff2015facenet}
Schroff, F.; Kalenichenko, D.; and Philbin, J. 2015.
\newblock Facenet: A unified embedding for face recognition and clustering.
\newblock In \emph{Proceedings of the IEEE conference on computer vision and
  pattern recognition}, 815--823.

\bibitem[{Tack et~al.(2020)Tack, Mo, Jeong, and
  Shin}]{DBLP:journals/corr/abs-2007-08176}
Tack, J.; Mo, S.; Jeong, J.; and Shin, J. 2020.
\newblock {CSI:} Novelty Detection via Contrastive Learning on Distributionally
  Shifted Instances.
\newblock \emph{CoRR}, abs/2007.08176.

\bibitem[{V{\"a}is{\"a}l{\"a}(2002)}]{vaisala2002isometric}
V{\"a}is{\"a}l{\"a}, J. 2002.
\newblock Isometric approximation property in euclidean spaces.
\newblock \emph{Israel Journal of Mathematics}, 128(1): 1--27.

\bibitem[{Vaisala(2002)}]{vaisala2002survey}
Vaisala, J. 2002.
\newblock A survey of nearisometries.
\newblock \emph{arXiv preprint math/0201098}.

\bibitem[{Vinyals et~al.(2014)Vinyals, Toshev, Bengio, and
  Erhan}]{DBLP:journals/corr/VinyalsTBE14}
Vinyals, O.; Toshev, A.; Bengio, S.; and Erhan, D. 2014.
\newblock Show and Tell: {A} Neural Image Caption Generator.
\newblock \emph{CoRR}, abs/1411.4555.

\bibitem[{Xuan et~al.(2020)Xuan, Stylianou, Liu, and Pless}]{xuan2020hard}
Xuan, H.; Stylianou, A.; Liu, X.; and Pless, R. 2020.
\newblock Hard negative examples are hard, but useful.
\newblock In \emph{European Conference on Computer Vision}, 126--142. Springer.

\bibitem[{Zhou et~al.(2019)Zhou, Niu, Wang, Gao, Zhang, and
  Hua}]{DBLP:journals/corr/abs-1911-07528}
Zhou, M.; Niu, Z.; Wang, L.; Gao, Z.; Zhang, Q.; and Hua, G. 2019.
\newblock Ladder Loss for Coherent Visual-Semantic Embedding.
\newblock \emph{CoRR}, abs/1911.07528.

\end{thebibliography}

\onecolumn
\pagebreak

\setcounter{equation}{27}

\section{Appendix A. Proof of Theorem 3}


\begin{proof}
Here, we present a detailed proof for the theorem using equations (\ref{eq:23}-\ref{eq:27})

From the definition of $d_\textrm{iso}$ in (\ref{eq:23}), we introduce the anchor, positive, and negative triplet $(x,x^+,x^-)$ in (\ref{eq:24}) by re-labelling $x_1\to x$. Recognizing that $x_2$ must either have the same or different label from $x_1$, we re-label $x_2\to x^+$ or $x_2\to x^-$, and pick the max of these distances for any given triplet. 
\begin{gather}
    d_\textrm{iso}(f_\theta,\mathcal F_{TS}) = \inf_{f^*\in\mathcal F_{TS}}\sup_{x_1,x_2\in\mathcal X}\bigg| ||f_\theta(x_1)-f_\theta(x_2)|| - ||f^*(x_1) - f^*(x_2)||\bigg| \label{eq:23} \\
    =\inf_{f^*\in\mathcal F_{TS}}\sup_{x,x^+,x^-}\max\left\{\bigg| ||f_\theta(x)-f_\theta(x^+)|| - ||f^*(x) - f^*(x^+)||\bigg|,\bigg| ||f_\theta(x)-f_\theta(x^-)|| - ||f^*(x) - f^*(x^-)||\bigg|\right\} \label{eq:24}
\end{gather}

Inequality (\ref{eq:25}) follows from $\max(a,b)\leq a+b$ for positive $a,b$.
\begin{gather}
    \leq \inf_{f^*\in\mathcal F_{TS}}\sup_{x,x^+,x^-} \bigg| ||f_\theta(x)-f_\theta(x^+)|| - ||f^*(x) - f^*(x^+)||\bigg|+\bigg| ||f_\theta(x)-f_\theta(x^-)|| - ||f^*(x) - f^*(x^-)||\bigg| \label{eq:25}
\end{gather}

Now fix the $f^*$ that minimizes (\ref{eq:25}). We next prove via contradiction that the first term (\ref{expr:di+}) is positive and the second term (\ref{expr:di-}) is negative. 
\begin{subequations}
    \begin{gather}
        ||f_\theta(x)-f_\theta(x^+)||-||f^*(x)-f^*(x^+)|| \label{expr:di+} \\
        ||f_\theta(x)-f_\theta(x^-)||-||f^*(x)-f^*(x^-)||\label{expr:di-}
    \end{gather}
\end{subequations}

There are four cases we must consider here, as we treat the zero case as either positive or negative. Case 1: (\ref{expr:di+}) is positive, (\ref{expr:di-}) is positive. Denoting this as $++$, our four cases are $(1:++)$, $(2:--)$, $(3:-+)$, $(4:+-)$. Now we prove by contradiction that case 4 is the only valid one.

Case $1 (++):$. Consider the function $f^\dagger(x)=(1+\delta)f^*(x)$ where $\delta>0$ is a small constant. Then $d_\textrm{iso}(f_\theta,f^\dagger)<d_\textrm{iso}(f_\theta,f^*)$, contradicting the statement that $f^*$ minimizes $d_\textrm{iso}$.

Case $2 (--):$. Consider the function $f^\dagger(x)=(1-\delta)f^*(x)$ where $\delta>0$ is a small constant. Then $d_\textrm{iso}(f_\theta,f^\dagger)<d_\textrm{iso}(f_\theta,f^*)$, contradicting the statement that $f^*$ minimizes $d_\textrm{iso}$.

Case $3 (-+):$. We can algebraically rearrange (\ref{eq:25}) to get:
\begin{gather}
    ||f^*(x)-f^*(x^+)||-||f^*(x)-f^*(x^-)||-||f_\theta(x)-f_\theta(x^+)|| + ||f_\theta(x)-f_\theta(x^-)|| \geq 0 \label{eq:thm3p0} \\
    ||f^*(x) - f^*(x^+)|| - ||f^*(x) - f^*(x^-)|| \leq 0 \label{eq:thm3p1} \\
    -\left(||f_\theta(x)-f_\theta(x^+)|| - ||f_\theta(x)-f_\theta(x^-)||\right) \geq 0 \label{eq:thm3p2}
\end{gather}

(\ref{eq:thm3p1}) comes from the definition of $f^*$ as a Triplet-Separated function; then (\ref{eq:thm3p2}) comes from combining (\ref{eq:thm3p0}) and (\ref{eq:thm3p1}). However, this means that the triplet that maximizes the expression has negative Triplet Loss, therefore there must be some other $f^*_2$ with a smaller value. This contradicts the statement that $f^*$ minimizes $d_\textrm{iso}$.

With Cases 1, 2, and 3 eliminated, we only have Case 4 and all the zero cases ($00$, $+0$, $-0$, $0+$, $0-$). We note that the cases $-0$ and $0+$ can be dis-proven using the same logic as Case 3.
This leaves the four following valid cases ($00$, $0-$, $+0$, $+-$), where we can connect back to (\ref{eq:25}) and write:
\begin{gather}
    \sup_{x,x^+,x^-} \bigg| ||f_\theta(x)-f_\theta(x^+)|| - ||f^*(x) - f^*(x^+)||\bigg|+\bigg| ||f_\theta(x)-f_\theta(x^-)|| - ||f^*(x) - f^*(x^-)||\bigg| \\
    = \sup_{x,x^+,x^-}||f_\theta(x)-f_\theta(x^+)|| - ||f_\theta(x)-f_\theta(x^-)|| - \left(||f^*(x) - f^*(x^+)|| - ||f^*(x) - f^*(x^-)||\right) \label{eq:26}
\end{gather}

Note that (\ref{eq:26}) resembles the Triplet Loss. The Triplet Loss for $f^*$ cannot dominate the maximum triplet loss for $f_\theta$, otherwise it would contradict the statement that $f^*$ minimizes the isometric error, giving us:
\begin{equation}
    -\left(||f^*(x)-f^*(x^+)|| - ||f^*(x)-f^*(x^-)||\right) \leq \sup_{x_2,x_2^+,x_2^-} ||f_\theta(x)-f_\theta(x^+)|| - ||f_\theta(x)-f_\theta(x^-)|| \label{eq:thm3p4} \\
\end{equation}


Using (\ref{eq:thm3p4}), we have the following relation with respect to (\ref{eq:26}).
\begin{equation}
    \leq 2\sup_{x,x^+,x^-} ||f_\theta(x)-f_\theta(x^+)|| - ||f_\theta(x)-f_\theta(x^-)|| \label{eq:27}
\end{equation}

Note that (\ref{eq:27}) is identical to twice the expression for the Triplet Loss sampled by Hard Negative Mining. Therefore the the Triplet Loss sampled by Hard Negative Mining upper bounds the isometric error by a constant factor of 2.

Additionally, we can prove that the Triplet Loss sampled by Hard Negative Mining upper bounds the isometric error. Starting from the definition of isometric error in (\ref{eq:30}), inequality (\ref{eq:31}) follows from $\max(a,b)\geq (a+b)/2$.

\begin{gather}
    d_\textrm{iso}(f_\theta,\mathcal F_{TS}) = \inf_{f^*\in\mathcal F_{TS}}\sup_{x_1,x_2\in\mathcal X}\bigg| ||f_\theta(x_1)-f_\theta(x_2)|| - ||f^*(x_1) - f^*(x_2)||\bigg| \label{eq:30} \\
    \geq \frac{1}{2}\inf_{f^*\in\mathcal F_{TS}}\sup_{x,x^+,x^-} \bigg| ||f_\theta(x)-f_\theta(x^+)|| - ||f^*(x) - f^*(x^+)||\bigg|+\bigg| ||f_\theta(x)-f_\theta(x^-)|| - ||f^*(x) - f^*(x^-)||\bigg| \label{eq:31}
\end{gather}

Once again fixing $f^*$, we have equality (\ref{eq:32}) by the same logic as the previous part. Inequality (\ref{eq:33}) follows from the fact that $||f^*(x) - f^*(x^+)|| - ||f^*(x) - f^*(x^-)||\leq 0$ by the definition of $f^*$ as Triplet-Separated.
\begin{gather}
    = \frac{1}{2}\sup_{x,x^+,x^-}||f_\theta(x)-f_\theta(x^+)|| - ||f_\theta(x)-f_\theta(x^-)|| - \left(||f^*(x) - f^*(x^+)|| - ||f^*(x) - f^*(x^-)||\right) \label{eq:32} \\
    \geq \frac{1}{2}\sup_{x,x^+,x^-}||f_\theta(x)-f_\theta(x^+)|| - ||f_\theta(x)-f_\theta(x^-)|| \label{eq:33}
\end{gather}

Therefore isometric error upper bounds the Triplet Loss sampled by Hard Negative Mining by a constant factor of 2.
\end{proof}

\end{document}